\documentclass{article}

\usepackage{nips14submit_e}
\usepackage{times}
\usepackage{hyperref}
\usepackage{natbib}
\usepackage{mathrsfs}
\usepackage{amsfonts}
\usepackage{algorithm}
\usepackage{algpseudocode}
\usepackage{amsthm}
\usepackage{amsmath}
\usepackage{amssymb}
\usepackage{graphicx}
\usepackage{comment}
\usepackage{bbm}
\usepackage{dsfont}
\usepackage{multicol}

\def\Rset{\mathbb{R}}
\def\E{\mathbb{E}}
\DeclareMathOperator*{\argmin}{\rm argmin}

\makeatletter
\newtheorem*{rep@theorem}{\rep@title}
\newcommand{\newreptheorem}[2]{%
\newenvironment{rep#1}[1]{%
 \def\rep@title{#2 \ref{##1}}%
 \begin{rep@theorem}}%
 {\end{rep@theorem}}}
\makeatother

\newcommand{\A}{\mathcal{A}}
\newcommand{\cT}{\mathscr{T}}
\newcommand{\n}{\mathfrak{n}}
\newcommand{\R}{\text{Reg}}
\newcommand{\Sur}{\text{Sur}}
\newcommand{\Ind}{\mathds{1}}
\newcommand{\EQ}{\gets}
\newcommand{\AND}{\mbox{{\bf and }}}

\newcommand{\ignore}[1]{}

\newtheorem{lemma}{Lemma}
\newtheorem{theorem}{Theorem}
\newreptheorem{theorem}{Theorem}
\newtheorem{proposition}{Proposition}
\newtheorem{corollary}{Corollary}
\newtheorem{definition}{Definition}

\hypersetup{
  colorlinks   = true, 
  urlcolor     = blue, 
  linkcolor    = blue, 
  citecolor   = blue 
}

\nipsfinalcopy

\begin{document}

\title{Revenue Optimization in Posted-Price Auctions with Strategic Buyers}

\author{
{\bf Mehryar Mohri} \\
Courant Institute
and Google Research\\
251 Mercer Street\\
New York, NY 10012\\
\texttt{\small mohri@cims.nyu.edu} 
\And
{\bf Andres Mu\~noz Medina} \\
Courant Institute\\
251 Mercer Street\\
New York, NY 10012\\
\texttt{\small munoz@cims.nyu.edu} 
}

\maketitle

\begin{abstract}

  We study revenue optimization learning algorithms for posted-price
  auctions with strategic buyers. We analyze a very broad family of
  monotone regret minimization algorithms for this problem, which
  includes the previously best known algorithm, and show that no
  algorithm in that family admits a strategic regret more favorable
  than $\Omega(\sqrt{T})$. We then introduce a new algorithm that
  achieves a strategic regret differing from the lower bound only by a
  factor in $O(\log T)$, an exponential improvement upon the previous
  best algorithm. Our new algorithm admits a natural analysis and
  simpler proofs, and the ideas behind its design are general. We also
  report the results of empirical evaluations comparing our algorithm
  with the previous state of the art and show a consistent exponential
  improvement in several different scenarios. 

\end{abstract}

\section{Introduction}

Auctions have long been an active area of research in Economics and
Game Theory \citep{vickrey2012, Milgrom1982,ostrovsky2011reserve}. In
the past decade, however, the advent of online advertisement has
prompted a more algorithmic study of auctions,
including the design of learning algorithms for revenue maximization for generalized
second-price auctions or second-price auctions with reserve
\citep{Cesa-BianchiGentileMansour2013,mohrimunoz2014,DiWeiLiWei14}.

These studies have been largely motivated by the widespread use of
AdExchanges and the vast amount of historical data thereby
collected -- AdExchanges are advertisement selling platforms
  using second-price auctions with reserve price to allocate
  advertisement space. Thus far, the learning algorithms proposed for
revenue maximization in these auctions critically rely on the
assumption that the bids, that is, the outcomes of auctions, are drawn
i.i.d.\ according to some unknown distribution. However, this
assumption may not hold in practice. In particular, with the knowledge
that a revenue optimization algorithm is being used, an advertiser
could seek to mislead the publisher by under-bidding.  In fact,
consistent empirical evidence of strategic behavior by advertisers has
been found by \citet{Edelman07}.  This motivates
the analysis presented in this paper of the interactions between
sellers and \emph{strategic buyers}, that is, buyers that may act
non-truthfully with the goal of maximizing their surplus.

The scenario we consider is that of \emph{posted-price auctions},
which, albeit simpler than other mechanisms, in fact matches a common
situation in AdExchanges where many auctions admit a single bidder. In
this setting, second-price auctions with reserve are equivalent to
posted-price auctions: a seller sets a reserve price for a good and
the buyer decides whether or not to accept it (that is to bid higher
than the reserve price). In order to capture the buyer's strategic
behavior, we will analyze an online scenario: at each time $t$, a
price $p_t$ is offered by the seller and the buyer must decide to
either accept it or leave it. This scenario can be modeled as a
two-player repeated non-zero sum game with incomplete information,
where the seller's objective is to maximize his revenue, while the
advertiser seeks to maximize her surplus as described in more detail
in Section~\ref{sec:setup}.

The literature on non-zero sum games is very rich
\citep{nachbar1997,NachbarBayesian, morris1994}, but much of
the work in that area has focused on characterizing different types of
equilibria, which is not directly relevant to the algorithmic
questions arising here. Furthermore, the problem we consider admits a
particular structure that can be exploited to design efficient revenue
optimization algorithms.

From the seller's perspective, this game can also be viewed as a
bandit problem \citep{kuleshov2010, robbins1985} since only the
revenue (or reward) for the prices offered is accessible to the
seller. \citet{KleinbergLeighton} precisely
studied this continuous bandit setting under the assumption of an
oblivious buyer, that is, one that does not exploit the seller's
behavior (more precisely, the authors assume that at each round the
seller interacts with a different buyer). The authors presented a
tight regret bound of $\Theta(\log \log T)$ for the scenario of a
buyer holding a fixed valuation and a regret bound of
$O(T^{\frac{2}{3}})$ when facing an adversarial buyer by using an
elegant reduction to a discrete bandit problem. However, as argued by
\cite{KAU}, when dealing with a \emph{strategic buyer}, the usual
definition of regret is no longer meaningful. Indeed, consider the
following example: let the valuation of the buyer be given by $v \in
[0,1]$ and assume that an algorithm with sublinear regret such as Exp3
\citep{AuerBianchiFreund} or UCB \citep{AuerBianchiFischer} is used
for $T$ rounds by the seller. A possible strategy for the buyer,
knowing the seller's algorithm, would be to accept prices only if they
are smaller than some small value $\epsilon$, certain that the seller
would eventually learn to offer only prices less than $\epsilon$. If
$\epsilon \ll v$, the buyer would considerably boost her
surplus while, in theory, the seller would have not incurred a large
regret since in hindsight, the best fixed strategy would have been to
offer price $\epsilon$ for all rounds. This, however is clearly not
optimal for the seller.
The stronger notion of policy regret introduced by \citet{Arora} has
been shown to be the appropriate one for the analysis of bandit
problems with adaptive adversaries. However, for the example just
described, a sublinear policy regret can be similarly achieved. Thus,
this notion of regret is also not the pertinent one for the study of
our scenario.  

We will adopt instead the definition of \emph{strategic-regret}, which
was introduced by \citet{KAU} precisely for the study of this problem.
This notion of regret also matches the concept of \emph{learning loss}
introduced by \citep{Agrawal} when facing an oblivious
adversary. Using this definition, \citet{KAU} presented both upper and
lower bounds for the regret of a seller facing a strategic buyer and
showed that the buyer's surplus must be discounted over time in order
to be able to achieve sublinear regret (see
Section~\ref{sec:setup}). However, the gap between the upper and lower
bounds they presented is in $O(\sqrt T)$.  In the following, we
analyze a very broad family of monotone regret minimization algorithms
for this problem (Section~\ref{sec:monotone}), which includes the
algorithm of \citet{KAU}, and show that no algorithm in that family
admits a strategic regret more favorable than $\Omega(\sqrt{T})$.
Next, we introduce a nearly-optimal algorithm that achieves a
strategic regret differing from the lower bound at most by a factor in
$O(\log T)$ (Section~\ref{sec:optimal}). This represents an
exponential improvement upon the existing best algorithm for this
setting. Our new algorithm admits a natural analysis and simpler
proofs.  A key idea behind its design is a method deterring the buyer
from \emph{lying}, that is rejecting prices below her valuation.

\section{Setup}
\label{sec:setup}

We consider the following game played by a buyer and a seller. A
good, such as an advertisement space, is repeatedly offered for sale by
the seller to the buyer over $T$ rounds. The buyer holds a private
valuation $v \in [0, 1]$ for that good.  At each round $t = 1, \ldots,
T$, a price $p_t$ is offered by the seller and a decision $a_t \in
\{0, 1\}$ is made by the buyer. $a_t$ takes value $1$ when the buyer
accepts to buy at that price, $0$ otherwise. We will say that a buyer
\emph{lies} whenever $a_t = 0$ while $p_t < v$. At the beginning of
the game, the algorithm $\A$ used by the seller to set prices is
announced to the buyer. Thus, the buyer plays strategically against this
algorithm. The knowledge of $\A$ is a standard assumption in mechanism
design and also matches the practice in AdExchanges.

For any $\gamma \in (0, 1)$, define the discounted
surplus of the buyer as follows: 
\begin{equation}
\label{eq:surplus}
  \Sur(\A, v) = \sum_{t = 1}^T \gamma^{t-1}a_t(v - p_t).
\end{equation}
The value of the \emph{discount factor} $\gamma$ indicates the
strength of the preference of the buyer for current surpluses versus
future ones. The performance of a seller's algorithm is measured by
the notion of \emph{strategic-regret} \citep{KAU} defined as follows:
\begin{equation}
\label{eq:regret}
  \R(\A, v) = T v - \sum_{t=1}^T a_t p_t.
\end{equation}
The buyer's objective is to maximize his discounted surplus, while
the seller seeks to minimize his regret. Note that, in view of
the discounting factor $\gamma$, the buyer is not fully
adversarial. The problem consists of designing algorithms achieving
sublinear strategic regret (that is a regret in $o(T)$).

The motivation behind the definition of strategic-regret is
straightforward: a seller, with access to the buyer's valuation, can
set a fixed price for the good $\epsilon$ close to this value. The
buyer, having no control on the prices offered, has no option but to
accept this price in order to optimize his utility. The revenue per
round of the seller is therefore $v - \epsilon$. Since there is no
scenario where higher revenue can be achieved, this is a natural
setting to compare the performance of our algorithm.

To gain more intuition about the problem, let us examine some of the
complications arising when dealing with a strategic buyer.  Suppose
the seller attempts to \emph{learn} the buyer's valuation $v$ by
performing a binary search. This would be a natural algorithm when
facing a truthful buyer. However, in view of the buyer's knowledge of
the algorithm, for $\gamma \gg 0$, it is in her best interest to lie
on the initial rounds, thereby quickly, in fact exponentially,
decreasing the price offered by the seller. The seller would then incur an
$\Omega(T)$ regret. A binary search approach is therefore ``too
aggressive''. Indeed, an untruthful buyer can manipulate the seller
into offering prices less than $v/2$ by lying about her value even just
once! This discussion suggests following a more conservative
approach. In the next section, we discuss a natural family of
conservative algorithms for this problem.

\section{Monotone algorithms}
\label{sec:monotone}

The following conservative pricing strategy was introduced by
\citet{KAU}. Let $p_1 = 1$ and $\beta < 1$. If price $p_t$ is rejected
at round $t$, the lower price $p_{t + 1} = \beta p_t$ is offered at
the next round. If at any time price $p_t$ is accepted, then this
price is offered for all the remaining rounds. We will denote this
algorithm by \texttt{monotone}. The motivation behind its design
is clear: for a suitable choice of $\beta$, the seller can slowly
decrease the prices offered, thereby pressing the buyer to reject many
prices (which is not convenient for her) before obtaining a favorable
price. The authors present an $O(T_\gamma \sqrt{T})$ regret bound for
this algorithm, with $T_\gamma = 1/(1 - \gamma)$. A more careful
analysis shows that this bound can be further tightened to
$O(\sqrt{T_\gamma T} + \sqrt{T})$ when the discount factor $\gamma$
is known to the seller.

Despite its sublinear regret, the \texttt{monotone}
algorithm remains sub-optimal for certain choices of $\gamma$. Indeed,
consider a scenario with $\gamma \ll 1$. For this setting, the buyer
would no longer have an incentive to lie, thus, an algorithm such as
binary search would achieve logarithmic regret, while the regret
achieved by the \texttt{monotone} algorithm is only guaranteed to be
in $O(\sqrt{T})$.

One may argue that the \texttt{monotone} algorithm is too specific
since it admits a single parameter $\beta$ and that perhaps a more
complex algorithm with the same monotonic idea could achieve a more
favorable regret. Let us therefore analyze a generic monotone
algorithm $\A_m$ defined by Algorithm~\ref{alg:monotone}. 

\begin{figure*}[ttt!]
\begin{minipage}[t]{.48 \textwidth}
\begin{algorithm}[H]
\caption{\small{Family of monotone algorithms.}}
\label{alg:monotone}
\begin{algorithmic}
\State \small{Let $p_1 = 1$ and $p_t \leq p_{t - 1}$ for $t = 2, \ldots T$}.
\State \small{$t \EQ 1$}
\State \small{$p \EQ p_t$}
\State \small{Offer price $p$}
\While{\small{(Buyer rejects $p$)} \AND \small{($t < T$)}}
\State \small{$t \EQ t + 1$}
\State \small{$p \EQ p_t$}
\State \small{Offer price $p$}
\EndWhile
\While{\small{($t < T$)}}
\State \small{$t \EQ t + 1$}
\State \small{Offer price $p$}
\EndWhile
\end{algorithmic}
\end{algorithm}
\end{minipage}
\hfill
\begin{minipage}[t]{.48 \textwidth}
\begin{algorithm}[H]
\caption{\small{Definition of $\A_r$.}}
\label{alg:reduction}
\begin{algorithmic}
\State \small{$\n = $ the root of $\cT(T)$}
\While{\small{Offered prices less than $T$}}
\State \small{Offer price $p_\n$}
\If{ \small{Accepted}}
\State \small{$\n = r(\n)$}
\Else
\State \small{Offer price $p_n$ for $r$ rounds}
\State \small{$\n = l(\n)$}
\EndIf
\EndWhile
\vspace{3em}
\end{algorithmic}
\end{algorithm}
\end{minipage}
\vskip -.15in
\end{figure*}

\begin{definition}
For any buyer's valuation $v \in [0, 1]$, define the \emph{acceptance time}
$\kappa^* = \kappa^*(v)$ as the first time a price offered
by the seller using algorithm $\A_m$ is accepted. 
\end{definition}

\begin{proposition}
\label{prop:monotonereg}
For any decreasing sequence of prices $(p_t)_{t = 1}^T$, there exists a
truthful buyer with valuation $v_0$  such that algorithm $\A_m$
suffers  regret of at least 
\begin{equation*}
  \R(\A_m, v_0) \geq \frac{1}{4}\sqrt{T - \sqrt{T}}.
\end{equation*}
\end{proposition}

\begin{proof}
By definition of the regret, we have $\R(\A_m, v) = v \kappa^* + (T -
\kappa^*)(v - p_{\kappa^*})$. We can consider two cases:
$\kappa^*(v_0) > \sqrt{T}$ for some $v_0 \in [1/2, 1]$ and
$\kappa^*(v) \leq\sqrt{T} $ for every $v \in [1/2, 1]$. In the former
case, we have $\R(\A_m, v_0) \geq v_0 \sqrt{T} \geq \frac{1}{2}
\sqrt{T}$, which implies the statement of the proposition. Thus,
we can assume the latter condition.

Let $v$ be uniformly distributed over $[\frac{1}{2}, 1]$. In view of
Lemma~\ref{lemma:expectation} (see Appendix~\ref{sec:monotonelower}), we have
\begin{align*}
\E[v \kappa^*]+ \E[(T - \kappa^*) (v - p_{\kappa^*})]
\geq \frac{1}{2} \E[\kappa^*] + (T - \sqrt{T}) \E[(v -
 p_{\kappa^*})]  
 \geq \frac{1}{2}\E[\kappa^*] + \frac{T - \sqrt{T}}{32
  \E[\kappa^*]}.
\end{align*}
The right-hand side is minimized for $\E[\kappa^*] = \frac{\sqrt{T -
    \sqrt{T}}}{4}$. Plugging in this value yields $\E[\R(\A_m, v)]
\geq \frac{\sqrt{T - \sqrt{T}}}{4}$, which implies the existence of
$v_0$ with $\R(\A_m, v_0) \geq \frac{\sqrt{T - \sqrt{T}}}{4}$.
\end{proof}
We have thus shown that any monotone algorithm $\A_m$ suffers a regret
of at least $\Omega(\sqrt{T})$, even when facing a truthful buyer. A
tighter lower bound can be given under a mild condition on the prices
offered.

\begin{definition}
  A sequence $(p_t)_{t = 1}^T$ is said to be \emph{convex} if it
  verifies $p_t - p_{t + 1} \geq p_{t + 1} - p_{t + 2}$ for $t = 1,
 \ldots, T - 2$.
\end{definition}
An instance of a convex sequence is given by the prices offered by the
\texttt{monotone} algorithm. A seller offering prices forming a
decreasing convex sequence seeks to control the number of lies
of the buyer by slowly reducing prices. The following
proposition gives a lower bound on the regret of any algorithm in this
family.

\begin{proposition}
\label{prop:lowerbound}
Let $(p_t)_{t = 1}^T$ be a decreasing convex sequence of prices. There
exists a valuation $v_0$ for the buyer such that the regret of the
monotone algorithm defined by these prices is $\Omega(\sqrt{T
C_\gamma} + \sqrt{T})$, where $C_\gamma = \frac{\gamma}{2(1 - \gamma)}$.
\end{proposition}
The full proof of this proposition is given in
Appendix~\ref{sec:monotonelower}. The proposition shows that when the
discount factor $\gamma$ is known, the \texttt{monotone} algorithm
is in fact asymptotically optimal in its class.

The results just presented suggest that the dependency on $T$ cannot
be improved by any monotone algorithm. In some sense, this family of
algorithms is ``too conservative''. Thus, to achieve a more favorable
regret guarantee, an entirely different algorithmic idea must be
introduced. In the next section, we describe a new algorithm
that achieves a substantially more advantageous
strategic regret by combining the fast convergence properties of a binary search-type
algorithm (in a truthful setting) with a method penalizing
untruthful behaviors of the buyer.

\section{A nearly optimal algorithm}
\label{sec:optimal}

Let $ \A$ be an algorithm for revenue optimization used against a
truthful buyer. Denote by $\cT(T)$ the tree associated to $\A$
after $T$ rounds. That is, $\cT(T)$ is a full tree of height $T$ with
nodes $\n \in \cT(T)$ labeled with the prices $p_\n$ offered by $
\A$.  The right and left children of $\n$ are denoted by $r(\n)$ and
$l(\n)$ respectively. The price offered when $p_\n$ is accepted by the
buyer is the label of $r(\n)$ while the price offered by $ \A$ if
$p_\n$ is rejected is the label of $l(\n)$. Finally, we will denote
the left and right subtrees rooted at node $\n$ by $\mathscr{L}(\n)$
and $\mathscr{R}(\n)$ respectively. Figure~\ref{fig:Kleintree} depicts
the tree generated by an algorithm proposed by
\citet{KleinbergLeighton}, which we will describe later.

Since the buyer holds a fixed valuation, we will consider algorithms that
increase prices only after a price is accepted and decrease it
only after a rejection. This is formalized in the following
definition.

\begin{definition}
An algorithm $ \A$ is said to be \emph{consistent} if $\max_{\n'
\in \mathscr{L}(\n)} p_{\n'} \leq p_n \leq \min_{\n' \in
\mathscr{R}(\n)} p_{\n'}$ for any node $\n \in \cT(T)$.
\end{definition}
For any consistent algorithm $ \A$, we define a modified algorithm
$\A_r$, parametrized by an integer $r \geq 1$, designed to face
strategic buyers. Algorithm $\A_r$ offers the same prices as $ \A$,
but it is defined with the following modification: when a price is
rejected by the buyer, the seller offers the same price for $r$
rounds. The pseudocode of $\A_r$ is given in
Algorithm~\ref{alg:reduction}. The motivation behind the modified
algorithm is given by the following simple observation: a strategic
buyer will lie only if she is certain that rejecting a price will
boost her surplus in the future.  By forcing the buyer to reject a
price for several rounds, the seller ensures that the future
discounted surplus will be negligible, thereby coercing the buyer
to be truthful.

\begin{figure}[t]
\centering
\begin{tabular}{cc}
\includegraphics[scale=.35]{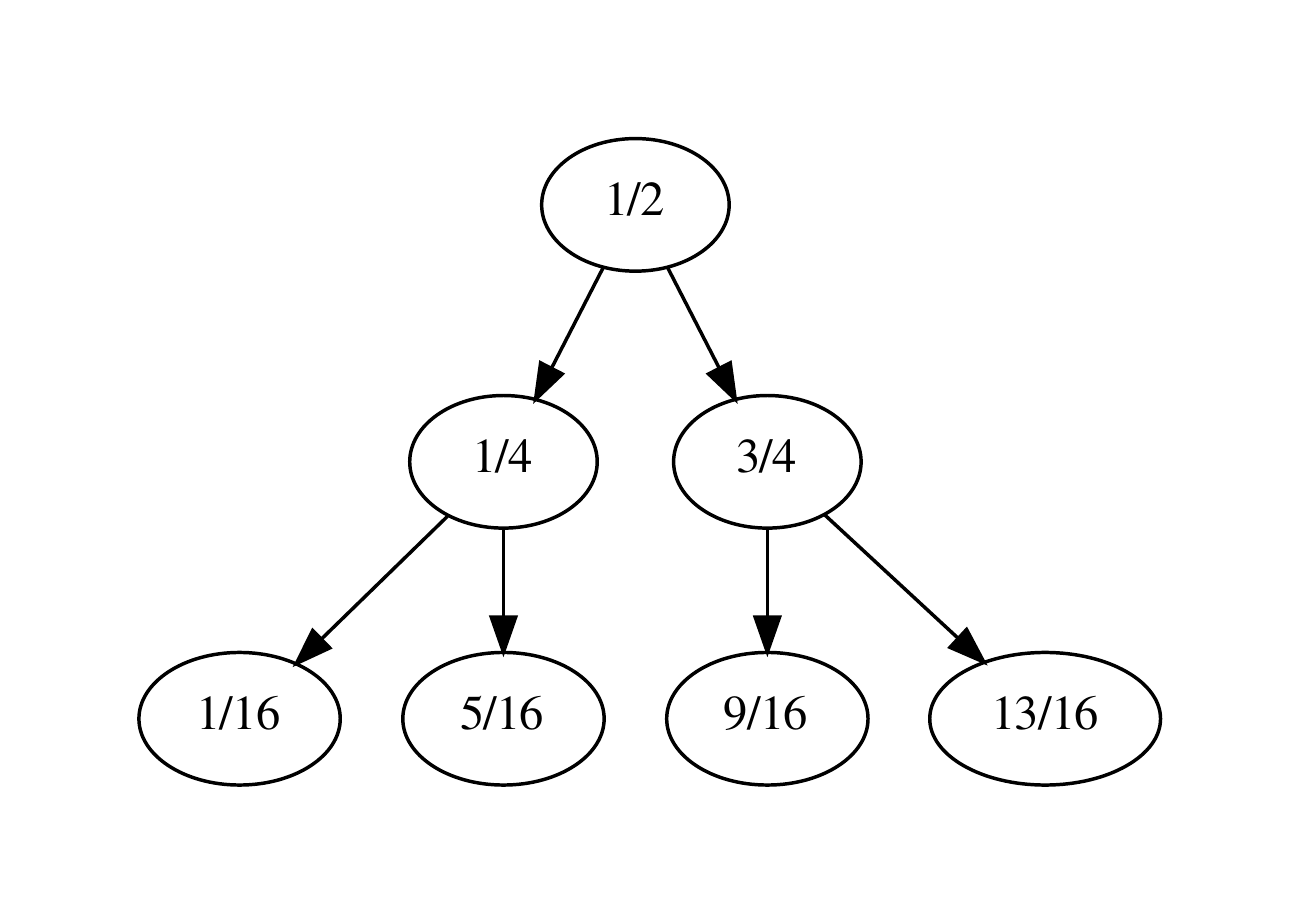} &
\includegraphics[scale=.35]{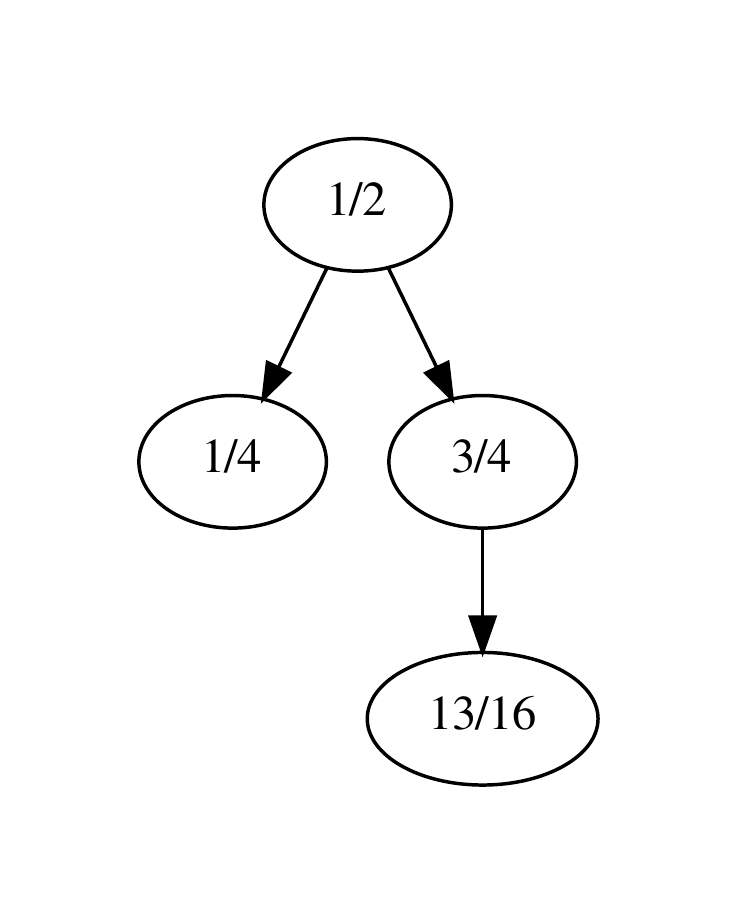}\\
(a) &  (b)
\end{tabular}
\caption{\small{(a) Tree $\cT(3)$ associated to the algorithm proposed in
\citep{KleinbergLeighton}. (b) Modified tree $\cT'(3)$ with $r = 2$.}}
\label{fig:Kleintree}
\end{figure}

We proceed to formally analyze algorithm $\A_r$. In particular, we
will quantify the effect of the parameter $r$ on the choice of the
buyer's strategy. To do so, a measure of the spread of the prices
offered by $\A_r$ is needed.

\begin{definition}
For any node $\n \in \cT(T)$ define the right increment of
$\n$ as $\delta_\n^r := p_{r(\n)} - p_\n$.  Similarly, define its left
increment to be $\delta_\n^l :=  \max_{\n' \in \mathscr{L}(\n)} p_{\n} - p_{\n'}$.
\end{definition}
The prices offered by $\A_r$ define a path in $\cT(T)$. For each node
in this path, we can define time $t(\n)$ to be the number of rounds
needed for this node to be reached by $\A_r$. Note that, since $r$ may
be greater than $1$, the path chosen by $\A_r$ might not necessarily
reach the leaves of $\cT(T)$.  Finally, let $\mathcal{S} \colon \n \mapsto
\mathcal{S}(\n)$ be the function representing the surplus obtained by
the buyer when playing an optimal strategy against $\A_r$ after node
$\n$ is reached.
 
\begin{lemma}
\label{lemma:recursion}
The function $\mathcal{S}$ satisfies the following recursive
relation:
  \begin{equation}
    \label{eq:recursion}
    \mathcal{S}(\n) = \max(\gamma^{t(\n) - 1}(v - p_\n)  + \mathcal{S}(r(\n)), \mathcal{S}(l(\n))).
  \end{equation}
\end{lemma}

\begin{proof}
Define a weighted tree $\cT'(T) \subset \cT(T)$ of nodes reachable by
algorithm $\A_r$. We assign weights to the edges in the
following way: if an edge on $\cT'(T)$ is of the form $(\n, r(\n))$,
its weight is set to be $\gamma^{t(\n) - 1} (v - p_\n)$, otherwise, it
is set to $0$. It is easy to see that the function $\mathcal{S}$
evaluates the weight of the longest path from node $\n$ to the
leafs of $\cT'(T)$. It thus follows from elementary graph algorithms
that equation \eqref{eq:recursion} holds.
\end{proof} 
The previous lemma immediately gives us necessary conditions for a
buyer to reject a price.

\begin{proposition}
\label{prop:rejcond}
For any reachable node $\n$, if price $p_\n$ is rejected by the buyer,
then the following inequality holds:
\begin{equation*}
v - p_\n < \frac{\gamma^{r}}{(1 - \gamma)(1 - \gamma^r)}(\delta_\n^l + \gamma \delta_\n^r).
\end{equation*}
\end{proposition}

\begin{proof}
A direct implication of Lemma~\ref{lemma:recursion} is that price
$p_{\n}$ will be rejected by the buyer if and only if
\begin{equation}
\label{eq:condition}
  \gamma^{t(\n) - 1}(v - p_\n) + \mathcal{S}(r(\n)) < \mathcal{S}(l(\n)).
\end{equation}
However, by definition, the buyer's surplus obtained by following any
path in $\mathscr{R}(\n)$ is bounded above by $\mathcal{S}(r(\n))$. In
particular, this is true for the path which rejects $p_{r(\n)}$ and
accepts every price afterwards. The surplus of this path is given by
$\sum_{t = t(\n) + r + 1}^T\gamma^{t-1} (v - \widehat p_t)$ where
$(\widehat p_t)_{t = t(\n) + r + 1}^T$ are the prices the seller would
offer if price $p_{r(\n)}$ were rejected. Furthermore, since algorithm
$\A_r$ is consistent, we must have $\widehat p_t \leq p_{r(\n)} = p_\n
+ \delta_\n^r$. Therefore, $\mathcal{S}(r(\n))$ can be bounded
as follows:
\begin{equation}
\label{eq:lowerboundR}
  \mathcal{S}(r(\n)) \geq \sum_{t = t(\n) + r + 1}^T \gamma^{t-1} (v - p_\n - \delta_\n^r) = \frac{\gamma^{t(\n) + r} - \gamma^T}{1 - \gamma}(v - p_\n - \delta_\n^r).
\end{equation}
We proceed to upper bound $\mathcal{S}(l(\n))$. Since $ p_\n -
p_n' \leq \delta_\n^l$ for all $\n' \in \mathscr{L}(\n)$, $v -
p_{\n'} \leq v - p_\n + \delta_\n^l$ and
\begin{equation}
  \label{eq:upperboundL}
\mathcal{S}(l(\n)) \leq \sum_{t = t_\n + r}^T \gamma^{t-1}(v - p_\n +
\delta_\n^l) 
= \frac{\gamma^{t(\n) + r - 1} - \gamma^T}{1 - \gamma}(v - p_\n + \delta_\n^l).
\end{equation}
Combining inequalities \eqref{eq:condition}, \eqref{eq:lowerboundR} and \eqref{eq:upperboundL} we
conclude that
\begin{align*}
\gamma^{t(\n) -1}(v - p_\n) + \frac{\gamma^{t(\n) + r} - \gamma^T}{1 - \gamma}(v - p_\n - \delta_\n^r) &\leq
 \frac{\gamma^{t(\n) + r - 1} - \gamma^T}{1 - \gamma}(v - p_\n + \delta_\n^l) \\
\Rightarrow \quad (v - p_\n)\left(1 + \frac{\gamma^{r + 1} - \gamma^r}{1 - \gamma} \right) &\leq
 \frac{\gamma^{r} \delta_\n^l + \gamma^{r + 1} \delta_\n^r - \gamma^{T - t(\n) + 1}(\delta_\n^r + \delta_\n^l)}{1 - \gamma} \\
\Rightarrow  \quad (v - p_n)(1 - \gamma^r) &\leq \frac{\gamma^r(\delta_\n^l + \gamma \delta_\n^r)}{1 - \gamma}.
\end{align*}
Rearranging the terms in the above inequality yields the desired result.
\end{proof}

Let us consider the following instantiation of algorithm $ \A$
introduced in \citep{KleinbergLeighton}.  The algorithm keeps track of
a \emph{feasible interval} $[a, b]$ initialized to $[0, 1]$ and an
increment parameter $\epsilon$ initialized to $ 1/2$. The algorithm
works in phases. Within each phase, it offers prices $a + \epsilon, a
+ 2 \epsilon, \ldots$ until a price is rejected. If price $a + k
\epsilon$ is rejected, then a new phase starts with the feasible
interval set to $[a + (k - 1) \epsilon, a + k \epsilon]$ and the
increment parameter set to $\epsilon^2$. This process continues until
$b - a < 1/T$ at which point the last phase starts and price $a$ is
offered for the remaining rounds.  It is not hard to see that the
number of phases needed by the algorithm is less than $\lceil \log_2
\log_2 T \rceil + 1$. A more surprising fact is that this algorithm
has been shown to achieve regret $O(\log \log T)$ when the seller
faces a truthful buyer. We will show that the modification $\A_r$ of
this algorithm admits a particularly favorable regret bound. We will
call this algorithm $\mathsf{PFS}_r$ (penalized fast search algorithm).

\begin{proposition}
\label{prop:regret}
For any value of $v \in [0, 1]$ and any $\gamma \in (0, 1)$, the
regret of algorithm $\mathsf{PFS}_r$ admits the following upper bound:
\begin{equation}
\label{eq:reggammabound}
  \R(\mathsf{PFS}_r, v) \leq (v r + 1)   (\lceil \log_2 \log_2 T \rceil +
    1) +  \frac{(1 + \gamma)\gamma^r T}{2(1 - \gamma)(1 - \gamma^r)}.
\end{equation}
\end{proposition}

Note that for $r = 1$ and $\gamma \rightarrow 0$ the
upper bound coincides with that of \citep{KleinbergLeighton}.

\begin{proof} 
Algorithm $\mathsf{PFS}_r$ can accumulate regret in two ways: the price offered 
$p_\n$ is rejected, in which case the regret is $v$, or the price is
accepted and its regret is $v - p_\n$.

Let $K = \lceil \log_2 \log_2 T \rceil + 1 $ be the number of phases
run by algorithm $\mathsf{PFS}_r$. Since at most $K$ different prices are
rejected by the buyer (one rejection per phase) and each price must be
rejected for $r$ rounds, the cumulative regret of all rejections is
upper bounded by $v K r$.

The second type of regret can also be bounded straightforwardly. For
any phase $i$, let $\epsilon_i$ and $[a_i, b_i]$ denote the
corresponding search parameter and feasible interval respectively. If
$v \in [a_i, b_i]$, the regret accrued in the case where the buyer accepts a
price in this interval is bounded by $b_i - a_i =
\sqrt{\epsilon_i}$. If, on the other hand $v \geq b_i$, then it
readily follows that $v - p_\n < v - b_i + \sqrt{\epsilon_i}$ for all
prices $p_\n$ offered in phase $i$. Therefore, the regret obtained in
acceptance rounds is bounded by
\begin{equation*}
\sum_{i=1}^K N_i \Big((v - b_i) \Ind_{v > b_i} + \sqrt{\epsilon_i} \Big) 
\leq \sum_{i=1}^K (v - b_i) \Ind_{v >  b_i} N_i+ K,
\end{equation*}
where $N_i \leq
\frac{1}{\sqrt{\epsilon_i}}$ denotes the number of prices offered during the $i$-th
round.

Finally, notice that, in view of the algorithm's definition, every
$b_i$ corresponds to a rejected price. Thus, by
Proposition~\ref{prop:rejcond}, there exist nodes $\n_i$ (not
necessarily distinct) such that $p_{\n_i} = b_i$ and
\begin{equation*}
  v - b_i = v - p_{\n_i} \leq \frac{\gamma^r}{(1 - \gamma) (1 -
    \gamma^r)} (\delta_{\n_i}^l + \gamma \delta_{n_i}^r).
\end{equation*}
It is immediate that $\delta_{\n}^r \leq 1/2$ and $\delta_{\n}^l \leq 1/2$
for any node $\n$, thus, we can write
\begin{align*}
  \sum_{i=1}^K (v - b_i) \Ind_{v > b_i} N_i
& \leq  \frac{\gamma^r(1 +  \gamma) }{2(1  - \gamma) (1 - \gamma^r)}
  \sum_{i=1}^K N_i
\leq \frac{\gamma^r(1 +
    \gamma) }{2 (1  - \gamma) (1 - \gamma^r)} T.
\end{align*}
The last inequality holds since at most $T$ prices are offered by our
algorithm. Combining the bounds for both regret types yields the
result.
\end{proof} 

When an upper bound on the discount factor $\gamma$ is known to the
seller, he can leverage this information and optimize upper bound
\eqref{eq:reggammabound} with respect to the parameter $r$.

\begin{theorem}
\label{th:optimregret} 
Let $1/2 < \gamma < \gamma_0 < 1$ and $r^* = \Big\lceil \argmin_{r \geq
1} r + \frac{\gamma_0^r T}{(1 - \gamma_0) (1 - \gamma_0^r)} \Big\rceil.$
For any $v \in [0,1]$, if $T > 4$, the regret of $\mathsf{PFS}_{r^*}$ satisfies
\begin{equation*}
  \R(\mathsf{PFS}_{r^*}, v) \leq (2 v \gamma_0 T_{\gamma_0}\log c T  + 1
    + v )(\log_2 \log_2 T + 1) + 4 T_{\gamma_0},
\end{equation*}
where $c = 4 \log 2$.
\end{theorem}

The proof of this theorem is fairly technical and is deferred to the
Appendix. The theorem helps us define conditions under
which logarithmic regret can be achieved. Indeed, if $\gamma_0 =
e^{-1/\log T}= O(1 - \frac{1}{\log T})$, using the inequality $e^{-x}
\leq 1 - x + x^2/2$ valid for all $x > 0$ we obtain
\begin{equation*}
  \frac{1}{1 - \gamma_0} \leq \frac{\log^2 T}{2 \log T - 1} \leq \log T.
\end{equation*}
It then follows from Theorem~\ref{th:optimregret} that
\begin{equation*}
  \R(\mathsf{PFS}_{r^*}, v) \leq (2 v\log T \log cT + 1
  + v) (  \log_2 \log_2 T + 1) + 4\log T. \\ 
\end{equation*}
Let us compare the regret bound given by Theorem~\ref{th:optimregret}
with the one given by \citet{KAU}. The above discussion shows that for
certain values of $\gamma$, an exponentially better regret can be
achieved by our algorithm. It can be argued that the knowledge of an
upper bound on $\gamma$ is required, whereas this is not needed for
the \texttt{monotone} algorithm. However, if $\gamma > 1 -
1/\sqrt{T}$, the regret bound on \texttt{monotone} is
super-linear, and therefore uninformative. Thus, in order to properly
compare both algorithms, we may assume that $\gamma < 1 -
1/\sqrt{T}$ in which case, by Theorem~\ref{th:optimregret},
the regret of our algorithm is $O(\sqrt{T} \log T)$ whereas only
linear regret can be guaranteed by the \texttt{monotone}
algorithm. Even under the more favorable bound of $O(\sqrt{T_\gamma T}
+ \sqrt{T})$, for any $\alpha < 1$ and $\gamma < 1 - 1/T^\alpha$, the
\texttt{monotone} algorithm will achieve regret $O(T^{\frac{\alpha +
1}{2}})$ while a strictly better regret $O(T^\alpha \log T \log \log
T)$ is attained by ours. 

\section{Lower bound}
\label{sec:lowerbound}

The following lower bounds have been derived in previous work.

\begin{theorem}[\citep{KAU}]
Let $\gamma > 0$ be fixed.  For any algorithm $\A$, there exists a
valuation $v$ for the buyer such that $\R(\A, v) \geq \frac{1}{12}
T_\gamma$.
\end{theorem}

This theorem is in fact given for the stochastic setting where
the buyer's valuation is a random variable taken from some fixed
distribution $\mathcal{D}$. However, the proof of the theorem selects
$\mathcal{D}$ to be a point mass, therefore reducing the scenario to a
fixed priced setting.

\begin{theorem}[ \citep{KleinbergLeighton}]
Given any algorithm $\A$ to be played against a truthful buyer, there
exists a value $v \in [0,1]$ such that $\R(\A, v) \geq C \log \log T$ for
some universal constant $C$.
\end{theorem}
Combining these results leads immediately to the following.

\begin{corollary}
\label{coro:lower}
Given any algorithm $\A$, there exists a buyer's valuation $v \in
[0,1]$ such that $\R(\A, v) \geq \max\Big(\frac{1}{12}T_\gamma,  C \log
\log T\Big)$, for a universal constant $C$.
\end{corollary}

We now compare the upper bounds given in the previous section with the
bound of Corollary~\ref{coro:lower}. For $\gamma > 1/2$, we have
$\R(\mathsf{PFS}_r, v) = O(T_\gamma \log T \log \log T)$. On the other
hand, for $\gamma \leq 1/2$, we may choose $r = 1$, in which case, by
Proposition~\ref{prop:regret}, $\R(\mathsf{PFS}_r, v) = O(\log \log
T)$. Thus, the upper and lower bounds match up to an $O(\log T)$
factor.

\section{Empirical results}

In this section, we present the result of simulations comparing the
\texttt{monotone} algorithm and our algorithm $\mathsf{PFS}_r$. The
experiments were carried out as follows: given a buyer's valuation
$v$, a discrete set of false valuations $\widehat v$ were selected out
of the set $\{.03, .06, \ldots, v \}$. Both algorithms were run
against a buyer making the seller believe her valuation is
$\widehat v$ instead of $v$. The value of $\widehat v$ achieving
the best utility for the buyer was chosen and the regret for both
algorithms is reported in Figure~\ref{fig:results}.

\begin{figure}[t]
\centering
  \begin{tabular}{cccc}
    \small{$\boldsymbol{\gamma = .85, \, v = .75}$}
      &\small{ $\boldsymbol{\gamma = .95, \, v = .75}$} 
      & \small{$\boldsymbol{\gamma = .75,  \, v = .25}$} 
     & \small{$\boldsymbol{\gamma = .80, \, v = .25}$}\\ 
    \includegraphics[scale=.28]{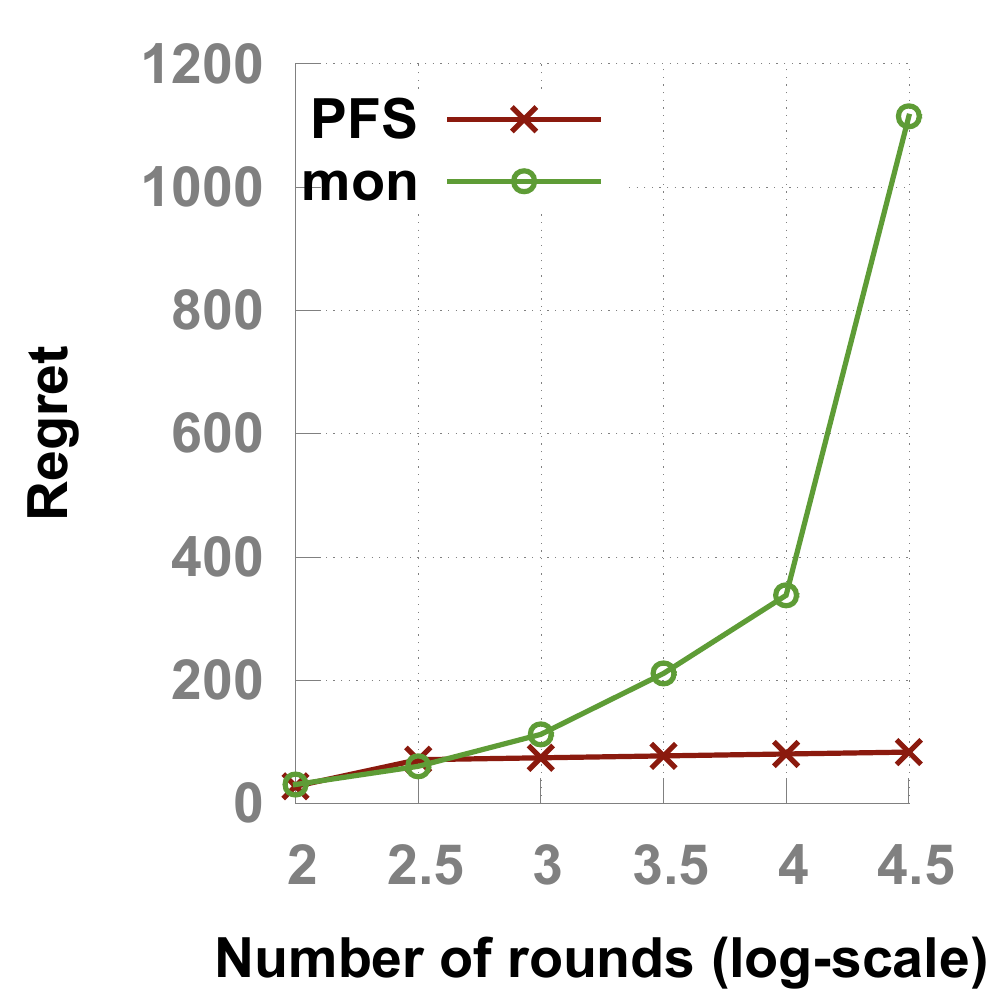}
    &\includegraphics[scale=.28]{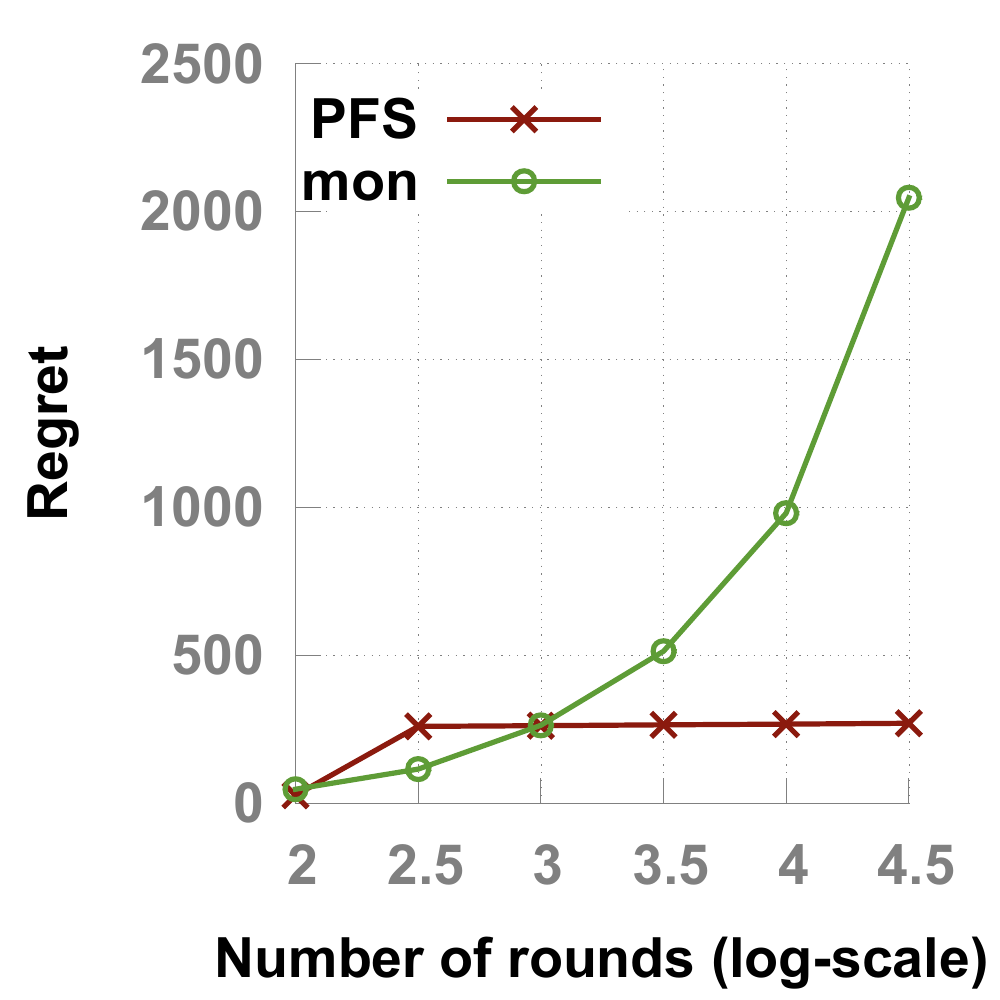}
    &\includegraphics[scale=.28]{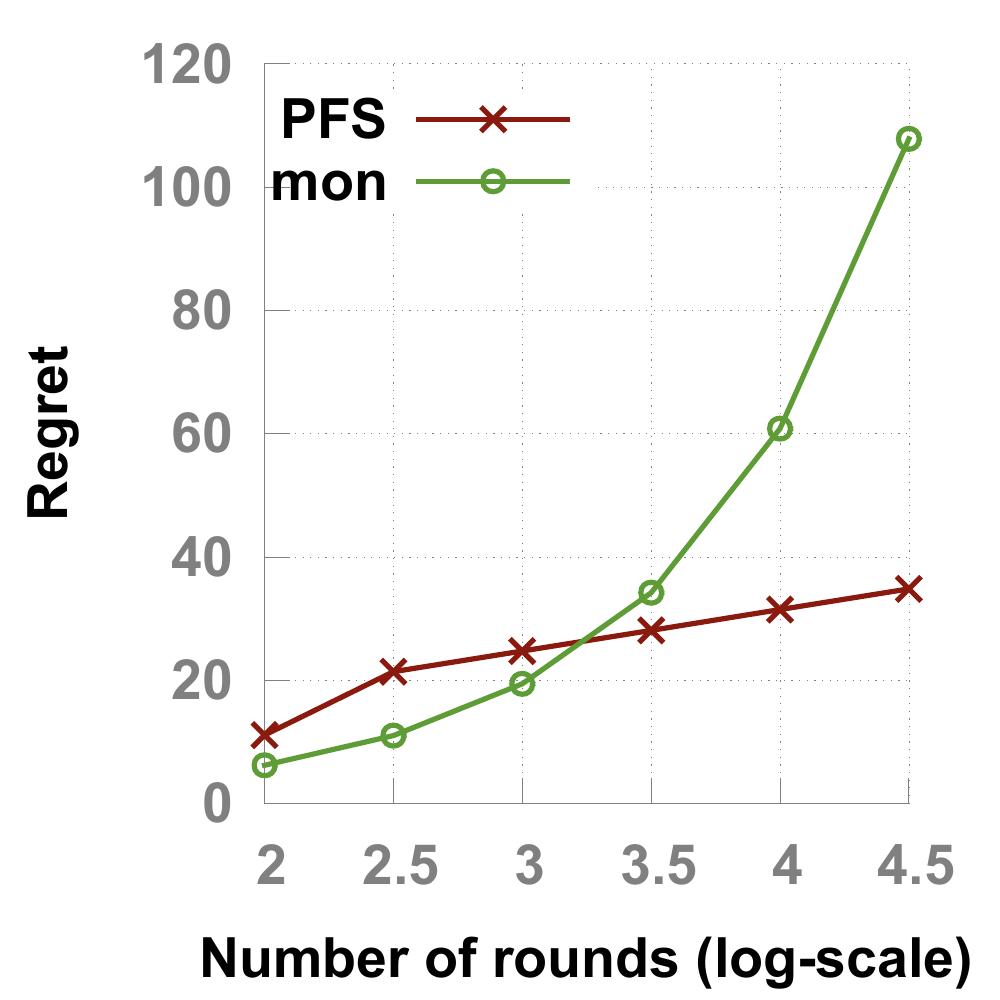}
    &\includegraphics[scale=.28]{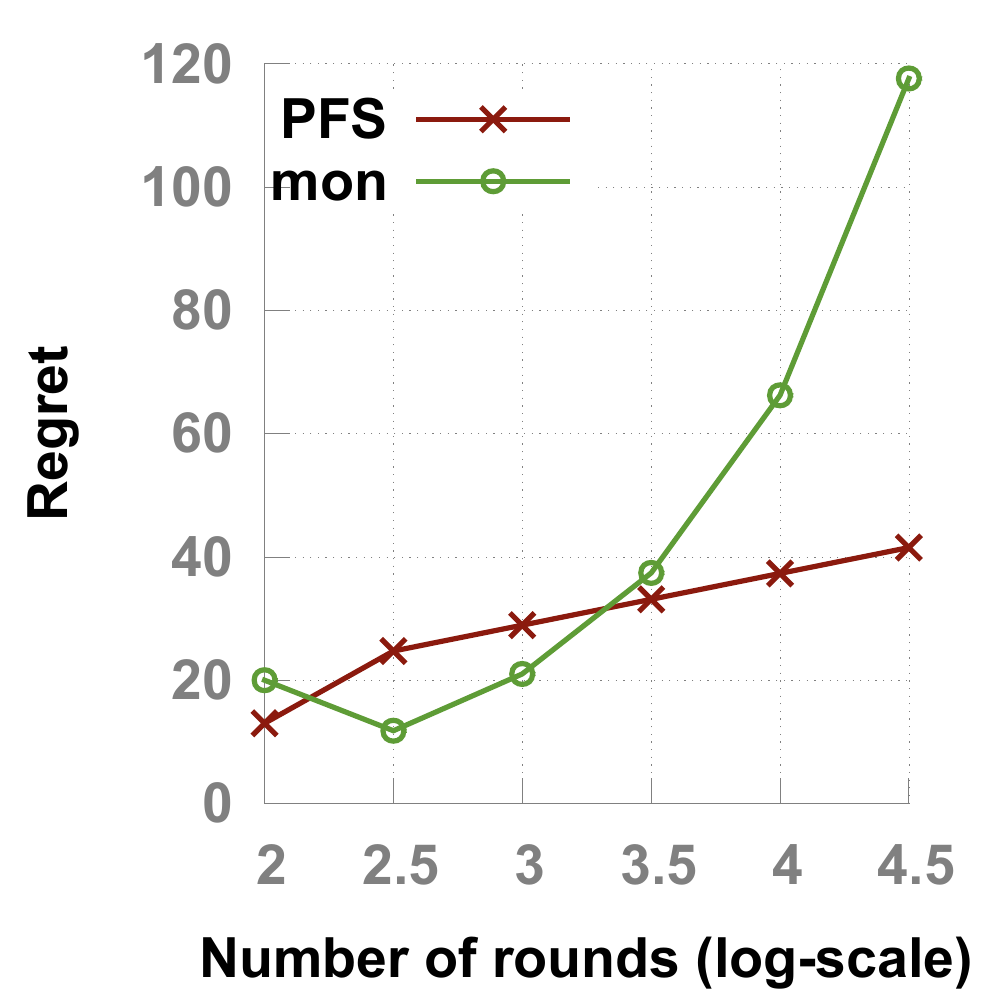}
  \end{tabular}
\vspace{-1em}
  \caption{\small{Comparison of the \texttt{monotone} algorithm and $\mathsf{PFS}_r$
    for different choices of $\gamma$ and $v$. 
The regret of each algorithm is plotted as a function of the number
rounds when $\gamma$ is not known to the algorithms
(first two figures) and when its value is made accessible to the algorithms
(last two figures).}}
\label{fig:results}
\vspace{-1em}
\end{figure}

We considered two sets of experiments. First, the value of parameter $\gamma$
 was left unknown to both algorithms and the
value of $r$ was set to $\log(T)$. This choice is motivated by the
discussion following Theorem~\ref{th:optimregret} since, for large
values of $T$, we can expect to achieve logarithmic regret. The first
two plots (from left to right) in Figure~\ref{fig:results} depict
these results. The apparent stationarity in the regret of
$\mathsf{PFS}_r$ is just a consequence of the scale of the plots as the
regret is in fact growing as $\log(T)$.
For the second set of experiments, we allowed access to the parameter
$\gamma$ to both algorithms.  The value of $r$ was chosen optimally
based on the results of Theorem~\ref{th:optimregret} and the parameter
$\beta$ of \texttt{monotone} was set to $1 - 1/\sqrt
  {TT_\gamma}$ to ensure regret in $O(\sqrt{T T_\gamma} + \sqrt{T})$.
It is worth noting that even though our algorithm was designed under
the assumption of some knowledge about the value of $\gamma$, the
experimental results show that an exponentially better performance
over the \texttt{monotone} algorithm is still attainable and in fact
the performances of the optimized and unoptimized versions of our
algorithm are comparable. A more comprehensive series of experiments
is presented in Appendix~\ref{sec:simulations}.

\section{Conclusion}

We presented a detailed analysis of revenue optimization algorithms
against strategic buyers. In doing so, we reduced the gap between
upper and lower bounds on strategic regret to a logarithmic
factor. Furthermore, the algorithm we presented is simple to analyze
and reduces to the truthful scenario in the limit of $\gamma
\rightarrow 0$, an important property that previous algorithms did not
admit. We believe that our analysis helps gain a deeper understanding
of this problem and that it can serve as a tool for studying more
complex scenarios such as that of strategic behavior in repeated
second-price auctions, VCG auctions and general market strategies.

\section*{Acknowledgments}

We thank Kareem Amin, Afshin Rostamizadeh and Umar Syed for several
discussions about the topic of this paper. This work was partly funded
by the NSF award IIS-1117591.

\newpage
\bibliographystyle{abbrvnat}
\bibliography{ref}

\newpage

\section{Appendix}
\label{sec:appendix}

\begin{lemma}
\label{lemma:decreasing}
The function $g: \gamma \mapsto \frac{\log \frac{1}{\gamma}}{1 - \gamma}$ is
decreasing over the interval $(0, 1)$.
\end{lemma}

\begin{proof}
This can be straightforwardly established:
\begin{equation*}
  g'(\gamma) = \frac{-\frac{1 - \gamma}{\gamma} + \log \frac{1}{\gamma}}{(1 - \gamma)^2}
= \frac{ \gamma \log\big(1 - \big[1 - \frac{1}{\gamma}\big]\big) - (1 - \gamma)}{\gamma (1 - \gamma)^2} < \frac{(1 - \gamma) - (1 - \gamma)}{\gamma (1 - \gamma)^2} = 0,
\end{equation*}
using the inequality $\log(1 - x) < -x$ valid for all $x < 0$.
\end{proof}

\begin{lemma}
\label{lemma:increasing}
Let $a \geq 0$ and let $g\colon D \subset \Rset \to [a, \infty)$ be a
decreasing and differentiable function. Then, the function $F\colon \Rset
\to \Rset$ defined by
\begin{equation*}
  F(\gamma) = g(\gamma) - \sqrt{g(\gamma)^2 - b}
\end{equation*}
is increasing for all values of $b \in [0,a]$. 
\end{lemma}

\begin{proof}
We will show that $F'(\gamma) \geq 0$ for all $\gamma \in D$. Since
$F' = g ' [1 - g (g^2 - b)^{-1/2}]$ and $g' \leq 0$ by
  hypothesis, the previous statement is equivalent to showing that
  $\sqrt{g^2 - b} \leq g$ which is trivially verified since $b \geq 0$. 
\end{proof}

\begin{reptheorem}{th:optimregret} 
Let $1/2 < \gamma < \gamma_0 < 1$ and $r^* = \Big\lceil \argmin_{r \geq
1} r + \frac{\gamma_0^r T}{(1 - \gamma_0) (1 - \gamma_0^r)} \Big\rceil.$
For any $v \in [0,1]$, if $T > 4$, the regret of $\mathsf{PFS}_{r^*}$ satisfies
\begin{equation*}
  \R(\mathsf{PFS}_{r^*}, v) \leq (2 v \gamma_0 T_{\gamma_0}\log c T  + 1
    + v )(\log_2 \log_2 T + 1) + 4 T_{\gamma_0},
\end{equation*}
where $c = 4 \log 2$.
\end{reptheorem}

\begin{proof}
  It is not hard to verify that the function $r \mapsto r +
  \frac{\gamma_0^rT}{(1 - \gamma_0) (1 - \gamma_0^r)}$ is convex and
  approaches infinity as $r \rightarrow \infty$. Thus, it admits a
  minimizer $\bar{r}^*$ whose explicit expression can be found by solving
  the following equation
\begin{equation*}
0 = \frac{d}{dr} \left( r + \frac{\gamma_0^r T}{(1 -
      \gamma_0) (1 - \gamma_0^r)}\right)
= 1 + \frac{\gamma_0^r T \log \gamma_0}{(1 - \gamma_0) (1- \gamma_0^r)^2}.
\end{equation*}
Solving the corresponding second-degree equation yields 
\begin{equation*}
  \gamma_0^{\bar{r}^*}  =  \frac{ 2 + \frac{T \log(1/\gamma_0)}{1  - \gamma_0} -
    \sqrt{\Big(2 + \frac{T \log(1/\gamma_0)}{1 - \gamma_0}\Big)^2 - 4}}{2} =: F(\gamma_0).
\end{equation*}
By Lemmas~\ref{lemma:decreasing} and \ref{lemma:increasing}, 
the function $F$ thereby defined is increasing. Therefore,
$\gamma_0^{\bar{r}^*} \leq \lim_{\gamma_0 \rightarrow 1} F(\gamma_0)$ and
\begin{align}
\label{eq:gammaR}
\gamma_0^{\bar{r}^*} 
\leq \frac{2 + T - \sqrt{(2 + T)^2 - 4}}{2}  
= \frac{4}{2(2 + T + \sqrt{(2 + T)^2 - 4})} \leq \frac{2}{T}. 
\end{align}
By the same argument, we must have $\gamma_0^{\bar{r}^*} \geq F(1/2)$, that is
\begin{align*}
\gamma_0^{\bar{r}^*} 
\geq F(1/2) 
& = \frac{2 + 2 T \log 2 - \sqrt{(2 + 2 T \log 2)^2 - 4}}{2} \\
& = \frac{4}{2( 2 + 2T \log 2 + \sqrt{(2 + 2 T \log 2)^2 - 4})} \\
& \geq \frac{2}{4 + 4 T \log 2} \geq \frac{1}{4 T \log 2}.
\end{align*}
Thus, 
\begin{equation}
\label{eq:Rstar}
  r^* 
= \lceil \bar{r}^* \rceil
\leq \frac{\log(1 / F(1/2))}{\log(1/\gamma_0)} + 1 
\leq \frac{\log(4 T \log 2)}{\log 1 / \gamma_0} + 1.
\end{equation}
Combining inequalities \eqref{eq:gammaR} and
\eqref{eq:Rstar} with \eqref{eq:reggammabound} gives
\begin{align*}
  \R(\mathsf{PFS}_{r^*}, v) &\leq \left( v \frac{\log(4 T \log 2)}{\log 1 / \gamma_0}
    + 1 +v\right) (\lceil \log_2 \log_2 T \rceil +1 ) + \frac{(1 +
      \gamma_0) T }{(1 - \gamma_0)(T - 2)} \\
& \leq (2 v \gamma_0 T_{\gamma_0} \log(cT) + 1+ v) (\lceil \log_2 \log_2 T
\rceil + 1) + 4T_{\gamma_0}, 
\end{align*}
using the inequality $\log(\frac{1}{\gamma}) \geq \frac{1 -
  \gamma}{2 \gamma}$ valid for all $\gamma \in (1/2, 1)$.
\end{proof}

\subsection{Lower bound for monotone algorithms}
\label{sec:monotonelower}

\begin{lemma}
\label{lemma:expectation}
Let $(p_t)_{t = 1}^T$ be a decreasing sequence of prices. Assume that
the seller faces a truthful buyer. Then, if $v$ is sampled uniformly
at random in the interval $[\frac{1}{2}, 1]$, the following inequality
holds: 
\begin{equation*}
\E[\kappa^*] \geq \frac{1}{32 \E[v - p_{\kappa^*}]}.
\end{equation*}
\end{lemma}

\begin{proof}
Since the buyer is truthful,  $\kappa^*(v) = \kappa$ if
and only if $v \in [p_\kappa, p_{\kappa - 1}]$. Thus, we can write
\begin{equation*}
 \E[v - p_{\kappa^*}]
= \sum_{\kappa  = 2}^{\kappa_{\max}} \E\Big[\Ind_{v\in [p_\kappa,
  p_{\kappa - 1} ]} (v - p_\kappa)\Big]
= \sum_{\kappa  = 2}^{\kappa_{\max}} \int_{p_\kappa}^{p_{\kappa - 1}}
\!\! (v - p_\kappa) \, dv
= \sum_{\kappa = 2}^{\kappa_{\max}} \frac{(p_{\kappa - 1} - p_{\kappa})^2}{2},
\end{equation*}
where $\kappa_{\max} = \kappa^*(\frac{1}{2})$. 
Thus, by the Cauchy-Schwarz inequality, we can write
\begin{align*}
\E\left[\sum_{\kappa = 2}^{\kappa^*} p_{\kappa -1} - p_\kappa\right] 
& \leq \E\left[  \sqrt{\kappa^* \sum_{\kappa =2}^{\kappa^*}(p_{\kappa -1} -  p_\kappa)^2}\right]\\
& \leq \E\left[\sqrt{\kappa^*  \sum_{\kappa=2}^{\kappa_{\max}}(p_{\kappa-1} - p_\kappa)^2}\right] \\
& = \E\left[\sqrt{2\kappa^*  \E[v - p_{\kappa^*}]}\right]  \\
& \leq \sqrt{\E[\kappa^*]}\sqrt{2 \E[v - p_\kappa^*]},
\end{align*}
where the last step holds by Jensen's inequality. In view of that, since $v > p_{\kappa^*}$, it
follows that:
\begin{equation*}
  \frac{3}{4} = \E[v] 
\geq \E[p_{\kappa^*}] = \E\left[\sum_{\kappa = 2}^{\kappa^*}p_{\kappa } - p_{\kappa -1}\right] + p_1 
\geq -\sqrt{\E[\kappa^*]} \sqrt{2 \E[v - p_{\kappa^*}]} + 1.
\end{equation*}
Solving for $\E[\kappa^*]$ concludes the proof. 
\end{proof}

The following lemma characterizes the value of $\kappa^*$ when facing
a strategic buyer. 

\begin{lemma}
\label{lemma:acceptance}
For any $v \in [0, 1]$, $\kappa^*$ satisfies $v - p_{\kappa^*} \geq
C_\gamma^{\kappa^*} (p_{\kappa^*} - p_{\kappa^* + 1})$ with
$C_\gamma^{\kappa^*} = \frac{\gamma - \gamma^{T - \kappa^* + 1}}{1 -
\gamma}$. Furthermore, when $\kappa^* \leq 1 + \sqrt{T_\gamma T}$ and
$T \geq T_\gamma + \frac{2\log(2/\gamma)}{\log(1/\gamma)}$,
$C_\gamma^{\kappa^*}$ can be replaced by the universal constant
$C_\gamma = \frac{\gamma}{2 (1 - \gamma)}$.
\end{lemma}

\begin{proof}
Since an optimal strategy is played by the buyer, the surplus obtained
by accepting a price at time $\kappa^*$ must be greater than the
corresponding surplus obtained when accepting the first price at time
$\kappa^* + 1$. It thus follows that:
\begin{align*}
  \sum_{t = \kappa^*}^T \gamma^{t -1} (v - p_{\kappa^*}) &\geq \sum_{t = \kappa^* + 1}^T \gamma^{t-1}(v - p_{\kappa^* + 1}) \\
\Rightarrow \gamma^{\kappa^* -1 } (v  - p_{\kappa^*}) & \geq \sum_{t =
  \kappa^* + 1}^T \gamma^{t - 1}(p_{\kappa^*} - p_{\kappa^* + 1}) 
= \frac{\gamma^{\kappa^*} - \gamma^T}{1 - \gamma}(p_{\kappa^*} - p_{\kappa^*+1}).  
\end{align*}
Dividing both sides of the inequality by $\gamma^{\kappa^* - 1}$
yields the first statement of the lemma. Let us verify the second
statement. A straightforward calculation shows that the conditions on $T$
imply $T - \sqrt{T T_\gamma} \geq
\frac{\log(2/\gamma)}{\log(1/\gamma)}$, therefore
\begin{equation*}
C_\gamma^{\kappa^*}
\geq \frac{\gamma  - \gamma^{T - \sqrt{T_\gamma T }}}{1 - \gamma} \\
\geq \frac{\gamma - \gamma^{\frac{\log(2/\gamma)}{\log(1/\gamma)}}}{1
  - \gamma} 
= \frac{\gamma - \frac{\gamma}{2}}{1 - \gamma} 
= \frac{\gamma}{2(1 - \gamma)}.
\end{equation*}
\end{proof}

\begin{proposition}
\label{prop:gammalower}
For any convex decreasing sequence $(p_t)_{t = 1}^T$,  if $T \geq
T_\gamma + \frac{2 \log(2/\gamma)}{\log(1/\gamma)}$, then there
exists a valuation $v_0 \in [\frac{1}{2}, 1]$ for the buyer such that
\begin{equation*}
  \R(\A_m, v_0) 
\geq \max\left(\frac{1}{8}\sqrt{T - \sqrt{T}}, \sqrt{C_\gamma \Big(T - \sqrt{T_\gamma T}\Big) \left(\frac{1}{2} - \sqrt{\frac{C_\gamma}{T}}\right)}\right) 
= \Omega(\sqrt{T} + \sqrt{C_\gamma T}).
\end{equation*}
\end{proposition}

\begin{proof}
In view of Proposition~\ref{prop:monotonereg}, we only need to verify
that there exists $v_0 \in [\frac{1}{2}, 1]$ such that  
\begin{equation*}
\R(\A_m, v_0)  \geq \sqrt{C_\gamma \Big(T - \sqrt{T_\gamma T}\Big) \left(\frac{1}{2} - \sqrt{\frac{C_\gamma}{T}}\right)}.
\end{equation*}
Let $\kappa_{\min} = \kappa^*(1)$, and $\kappa_{\max} =
\kappa^*(\frac{1}{2})$.  If $\kappa_{\min} > 1 + \sqrt{T_\gamma T}$,
then $\R(\A_m, 1) \geq 1 + \sqrt{T_\gamma T}$, from which the
statement of the proposition can be derived straightforwardly. Thus,
in the following we will only consider the case $\kappa_{\min} \leq 1
+ \sqrt{T_\gamma T}$.  Since, by definition, the inequality
$\frac{1}{2} \geq p_{\kappa_{\max}}$ holds, we can write
\begin{align*}
  \frac{1}{2} 
\geq p_{\kappa_{\max}} 
= \sum_{\kappa = \kappa_{\min} + 1}^{\kappa_{\max}} (p_\kappa -
p_{\kappa - 1}) + p_{\kappa_{\min}} 
\geq \kappa_{\max} (p_{\kappa_{\min} + 1} - p_{\kappa_{\min}}) + p_{\kappa_{\min}},
\end{align*}
where the last inequality holds by the convexity of the sequence and
the fact that $p_{\kappa_{\min}} - p_{\kappa_{\min}-1} \leq 0$. The
inequality is equivalent to $p_{\kappa_{\min}} - p_{\kappa_{\min}
+1}\geq \frac{p_{\kappa_{\min}} -
\frac{1}{2}}{\kappa_{\max}}$. Furthermore, by
Lemma~\ref{lemma:acceptance}, we have
\begin{align*}
\max_{v \in [\frac{1}{2}, 1]} \R(\A_m, v) 
& \geq \max \left(\kappa_{\max} , (T -
  \kappa_{\min}) (p_{\kappa_{\min}} - p_{\kappa_{\min} + 1}) \right) \\
& \geq \max \bigg(\kappa_{\max}, C_\gamma \frac{(T -
    \kappa_{\min})(p_{\kappa_{\min}} - \frac{1}{2})}{\kappa_{\max}}\bigg).
\end{align*}
The right-hand side is minimized for $\kappa_{\max} =
\sqrt{C_\gamma (T - \kappa_{\min})(p_{\kappa_{\min} } -
  \frac{1}{2})}$. Thus, there exists a valuation $v_0$ for which the
following inequality holds:
\begin{align*}
\R(A_m, v_0) 
&\geq \sqrt{C_\gamma (T - \kappa_{\min})\Big(p_{\kappa_{\min}} - \frac{1}{2}\Big)}
\geq \sqrt{C_\gamma\Big(T - \sqrt{T_\gamma T}\Big)\Big(p_{\kappa_{\min}} - \frac{1}{2}\Big)}.
\end{align*}
Furthermore, we can assume that $p_{\kappa_{\min}} \geq 1 -
\sqrt{\frac{C_\gamma}{T}}$ otherwise $\R(A_m, 1) \geq (T-1
)\sqrt{C_\gamma/T}$, which is easily seen to imply the desired lower
bound. Thus, there exists a valuation $v_0$ such that
\begin{equation*}
  \R(\A_m, v_0) \geq \sqrt{C_\gamma \Big(T - \sqrt{T_\gamma
    T}\Big) \left(\frac{1}{2} - \sqrt{\frac{C_\gamma}{T}}\right)},
\end{equation*}
which concludes the proof.
\end{proof}

\section{Simulations}
\label{sec:simulations}

Here, we present the results of more extensive simulations for
$\mathsf{PFS}_r$ and the \texttt{monotone} algorithm. Again, we
consider two different scenarios. Figure~\ref{fig:agn} shows the
experimental results for an agnostic scenario where the value of the
parameter $\gamma$ remains unknown to both algorithms and where the
parameter $r$ of $\mathsf{PFS}_r$ is set to $\log(T)$. The results
reported in Figure~\ref{fig:know} correspond to the second scenario
where the discounting factor $\gamma$ is known to the algorithms and
where the parameter $\beta$ for the \texttt{monotone} algorithm is set
to $1 - 1/\sqrt{TT_\gamma}$. The scale on the plots is logarithmic in
the number of rounds and in the regret.

\begin{figure}[t]
\centering
\includegraphics[scale=.55]{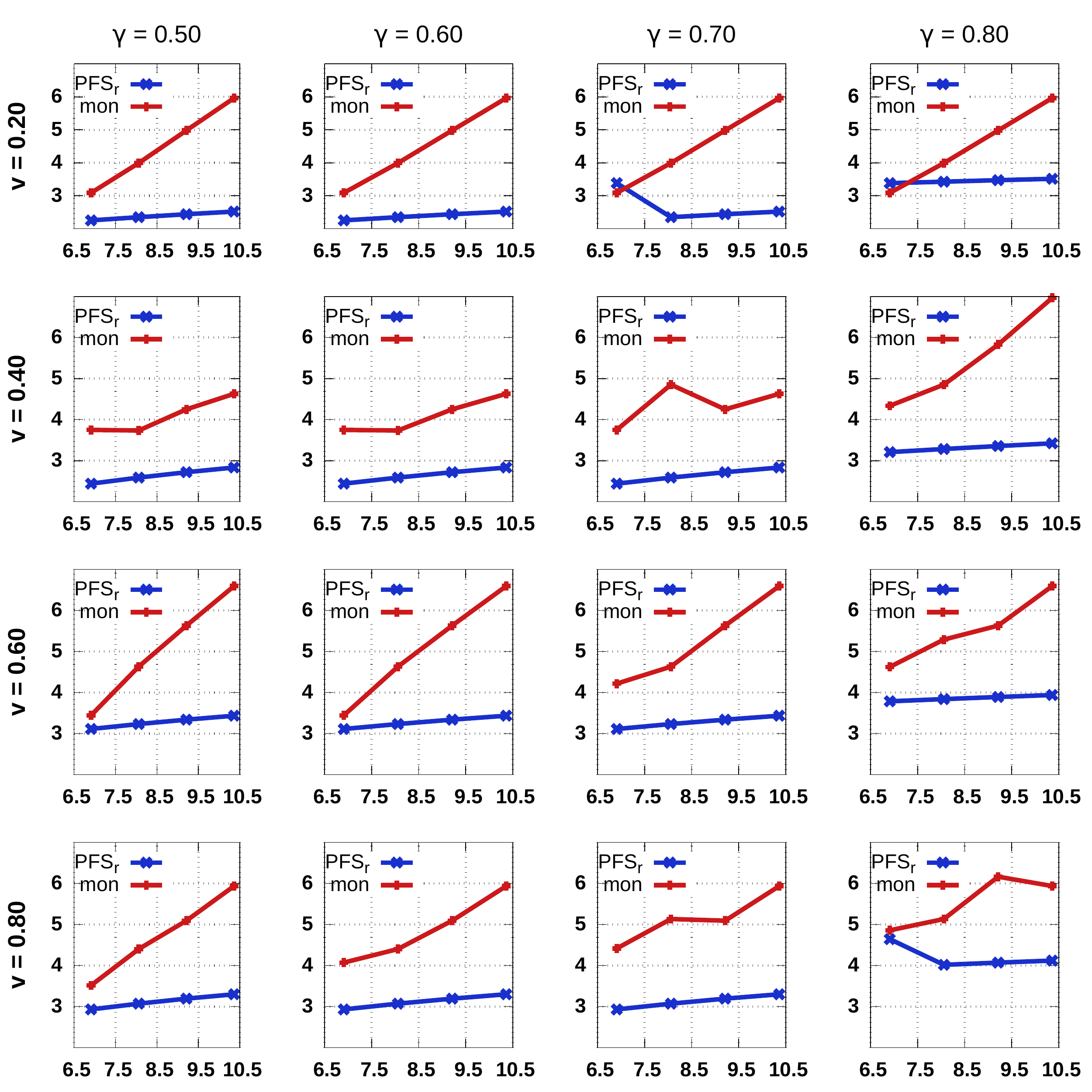}
\caption{Regret curves for $\mathsf{PFS}_r$ and \texttt{monotone} for
  different values of $v$ and $\gamma$. The value of $\gamma$ is
  not known to the algorithms.}
\label{fig:agn}
\end{figure}

\begin{figure}[t]
  \centering
  \includegraphics[scale=.55]{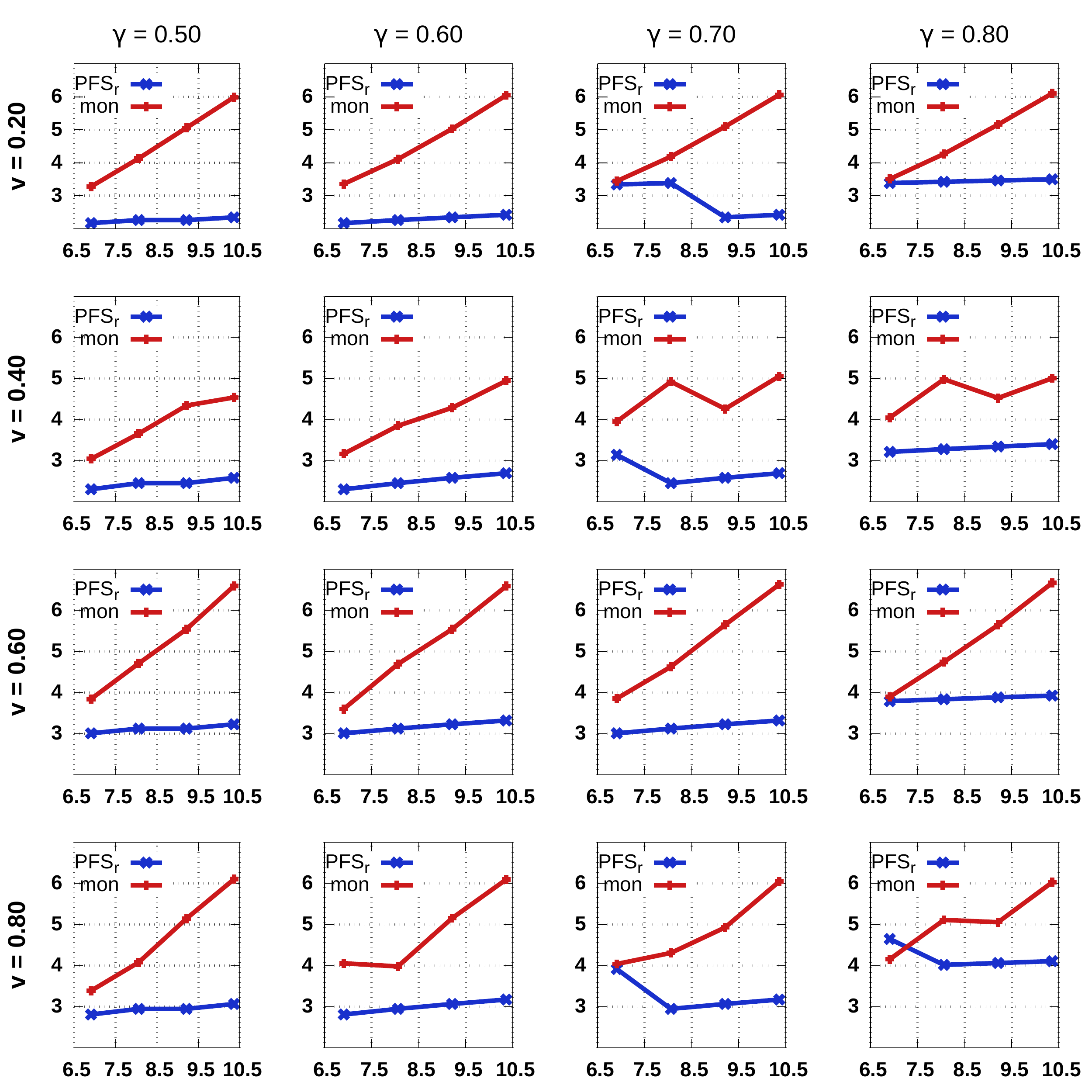}
  \caption{Regret curves for $\mathsf{PFS}_r$ and \texttt{monotone}
for different values of $v$ and $\gamma$. The value of $\gamma$ is
known to both algorithms.}
  \label{fig:know}
\end{figure}
\end{document}